\documentclass[letterpaper, 10 pt, conference]{ieeeconf}

\IEEEoverridecommandlockouts

\usepackage{amsmath} 
\usepackage{amssymb} 
\usepackage{xcolor}

\usepackage{amsthm}

\usepackage[noadjust]{cite}

\theoremstyle{plain}

\usepackage{makecell}
\newtheorem{theorem}{Theorem}
\newtheorem{lemma}[theorem]{Lemma}
\newtheorem{remark}{Remark}

\usepackage{graphicx}

\usepackage{subcaption}
\usepackage{hyperref}

\title{\LARGE \bf
Courteous MPC for Autonomous Driving with CBF-inspired \\Risk Assessment
}

\author{Yanze Zhang$^{1}$, Yiwei Lyu$^{2}$, Sude E. Demir$^{3}$, Xingyu Zhou$^{3}$ \\ Yupeng Yang$^{1}$, Junmin Wang$^{3}$, \textit{IEEE Fellow}, and Wenhao Luo$^{1}$
\thanks{$^{1}$Yanze Zhang, Yupeng Yang, and Wenhao Luo are with the Department of Computer Science, University of North Carolina at Charlotte, Charlotte, NC 28223, USA. Email: {\tt\small \{yzhang94, yyang52, wenhao.luo\}@charlotte.edu}.}
\thanks{$^{2}$Yiwei Lyu is with the Department of Electrical and Computer Engineering, Carnegie Mellon University,
        Pittsburgh, PA, 15213, USA. Email: {\tt\small yiweilyu@andrew.cmu.edu}.}
\thanks{$^{3}$Sude E. Demir, Xingyu Zhou, and Junmin Wang are with the Walker Department of Mechanical Engineering, University of Texas at Austin, Austin, TX, 78712, USA. Email: {\tt\small \{sude.demir, xingyu.zhou, jwang\}@austin.utexas.edu}.}
\thanks{This work was partially supported by NSF Awards 2312465 and 2312466.}
}

\begin{document}
\maketitle
\thispagestyle{empty}
\pagestyle{empty}

\begin{abstract}
With more autonomous vehicles (AVs) sharing roadways with human-driven vehicles (HVs), ensuring safe and courteous maneuvers that respect HVs' behavior becomes increasingly important. To promote both safety and courtesy in AV's behavior, an extension of Control Barrier Functions (CBFs)-inspired risk evaluation framework is proposed in this paper by considering both noisy observed positions and velocities of surrounding vehicles. The perceived risk by the ego vehicle can be visualized as a risk map that reflects the understanding of the surrounding environment and thus shows the potential for facilitating safe and courteous driving. By incorporating the risk evaluation framework into the Model Predictive Control (MPC) scheme, we propose a Courteous MPC for ego AV to generate courteous behaviors that 1) reduce the overall risk imposed on other vehicles and 2) respect the hard safety constraints and the original objective for efficiency. We demonstrate the performance of the proposed Courteous MPC via theoretical analysis and simulation experiments.
\end{abstract}

\section{Introduction} 
As more autonomous vehicles (AVs) are deployed on public roads, the coexistence of AVs and human-driven vehicles (HVs) is expected to lead to frequent interactions between AVs and HVs, as well as among AVs themselves. Although AVs are designed to engineer for a combination of safety, efficiency, and driving quality, the underlying ego-centric design often overlooks the impact of AV behaviors to the surrounding environments, 
which may result in unwelcome or even dangerous actions by AVs to other vehicles \cite{sun2018courteous, wang2019enabling}.

To realize safe and courteous driving behavior for AVs, 
some methods incorporate a courtesy term as a reward feature, and use Inverse Reinforcement Learning (IRL) to learn appropriate weights for such a feature from actual human driving data \cite{sun2018courteous, wang2023active, toghi2022social, lyu2022responsibility}. Courtesy terms that have been quantified include costs related to the human driver’s tasks \cite{sun2018courteous}, active information gathering \cite{wang2023active}, and measures of selfishness or altruism through Social Value Orientation \cite{toghi2022social, lyu2022responsibility}. Other methods integrate social factors with deep learning to foster socially compliant behaviors, e.g. Social GAN \cite{gupta2018social}.
However, the vehicle’s social preferences, subject to some specific scenario or dataset, may not perform well in unfamiliar conditions.
Furthermore, it is challenging to incorporate uncertainties into the generation of these methods or to estimate the failure probability in uncertain trafﬁc scenarios, resulting in a lack of theoretical safety guarantee.

\begin{figure}
    \centering
    \includegraphics[width=3.0 in]{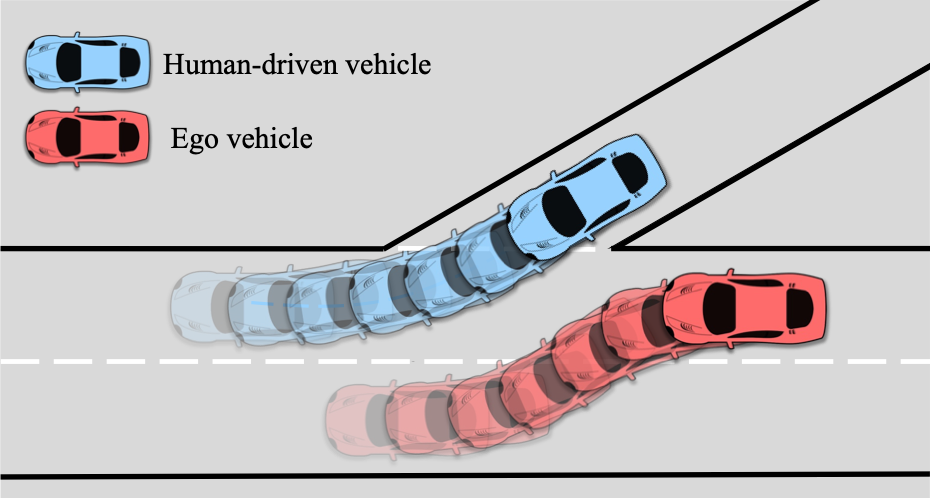}
    \caption{\small On-ramp merging scenario where the human-driven vehicle (blue) is trying to merge onto the main lane occupied by the ego vehicle (red). To accommodate the human-driven vehicle (HV) in a courteous manner, the ego vehicle could consider lane-change behavior that makes space for the HV without sacrificing much on safety and efficiency. 
    }
    \label{fig:real-world}
    \vspace{-0.5 cm}
\end{figure}

Some probabilistic approaches incorporate uncertainties into Partially Observable Markov Decision Process (POMDP) \cite{hubmann2018automated} for decision making.
Specifically, research in \cite{hsu2023interpretable} categorizes safety and courtesy into different levels and employs Hidden Markov Models (HMM) to effectively interpret the long-term patterns of behavior despite observational noise, enhancing trajectory prediction accuracy.
Other methods \cite{zhu2019chance, nyberg2021risk} formulate uncertainty-aware distance constraints as chance constraints and integrate them into the optimization problem to enhance safety in dynamic environments. However, enforcing these state-based constraints
often require enlarged bounding volumes to accommodate braking distances, potentially overestimating the probability and leading to conservative behaviors.

Benefiting from the less restrictive constraints using Control Barrier Functions (CBFs) \cite{ames2019control},
recent methods \cite{luo2020multi, lyu2021probabilistic,zhang2023occlusion, yang2024decentralized, zhou2024safety} reformulated the state-based chance constraints for collision avoidance into deterministic control constraints that enable safety guarantees with high probability as well as task efficiency.
Additionally, work in \cite{lyu2023risk} has introduced a risk evaluation framework inspired by CBF
\cite{lyu2023cbf}
under motion uncertainty and integrated it into a risk-aware controller to facilitate effective decision-making in complex dynamic environments.
Building on these foundations, we develop an extension to this risk evaluation framework that takes into account both noisy observed positions and motions, which are then integrated with a model predictive controller (MPC) to ensure smooth and robust decision-making for AVs. This integrated controller not only provides effective safety guarantees but also demonstrates courtesy by allowing additional space for the maneuvers of surrounding vehicles.
Our contributions in this paper are threefold:
\begin{itemize}
    \item A novel extension to the existing CBF-inspired risk evaluation framework is proposed, which incorporates both position error and motion error to robustify the risk measurement under uncertainty;
    \item A Courteous MPC method for autonomous driving is proposed that integrates the developed risk evaluation to consider
    heterogeneous surrounding human-driven vehicles with a provable probabilistic safety guarantee;
    \item Simulations on interaction with vehicles modeled by intelligent driver model (IDM)~\cite{treiber2000congested} and NGSIM dataset~\cite{alexiadis2004next} are shown to validate the effectiveness of the proposed risk evaluation framework and Courteous MPC.
\end{itemize}

\section{Background}
\subsection{Control Barrier Functions}
Consider a dynamical system in control-affine form:
\begin{equation} \label{eq:Affine}
\dot{x} = f(x) + g(x)u
\end{equation}
where $x \in \mathbb{R}^d$ and $u \in \mathbb{R}^m$ are the system state vector and control input vector, respectively. The vector fields 
$f: \mathbb{R}^d \mapsto \mathbb{R}^d$ and matrix-valued function $g: \mathbb{R}^d \mapsto \mathbb{R}^{d \times m}$ are assumed to be locally Lipschitz continuous.

Given the control affine system in Eq.~\eqref{eq:Affine} and the safe set $\mathcal{H}$ defined as the zero-superlevel set of a continuously differentiable function $h: \mathbb{R}^{d} \mapsto \mathbb{R}$, Safety Barrier Certificates \cite{ames2019control} has been used to enforce the safety set forward invariance, as defined in the following lemma. 

\begin{lemma}
\label{lem:cbf}[Summarized from \cite{ames2019control}]
Given a dynamical system affine defined in Eq.~\eqref{eq:Affine} and a safe set $\mathcal{H}$ as the 0-super level set of a continuously differentiable function $h: \mathbb{R}^d \mapsto \mathbb{R}$, the function $h$ is called a control barrier function, if there exists an extended class-$\mathcal{K}$ function $\alpha(\cdot)$ such that 
$\sup_{u\in \mathbb{R}^m}\{\dot{h}(x, u)\}\geq -\alpha(h(x))$ for all $x \in \mathbb{R}^d$. The admissible control space for any Lipschitz continuous controller $u \in \mathbb{R}^m$ rendering $\mathcal{H}$ forward invariant thus becomes:
{\small\begin{align}\label{eq:cbc_lemma}
    K_\textrm{cbf}(x) = \{ u \in \mathbb{R}^m | L_f h(x) + L_gh(x)u \geq -\alpha(h(x)) \}
\end{align}}
where $L_f h(x) = \nabla h^{T}(x) f(x)$ and
$L_g h(x) = \nabla h^{T}(x) g(x)$, respectively. 
\end{lemma}
This defines the set of controllers enforcing the system state $x$ staying in $\mathcal{H}$ over time if the initial system state  $x(t=0)$ is inside the set $\mathcal{H}$.

\subsection{Sample-Data Control Barrier Functions}
However, when the system states are updated discretely and so does the controller, such controller implemented under a discrete-time system may not satisfy the safety condition using the continuous constraints defined in Eq.~\eqref{eq:cbc_lemma} between sample time steps \cite{yang2020continuous}. To this end, the Sampled-data Control Barrier Functions (Sampled-data CBFs) \cite{breeden2021control} are introduced
to enforce the forward invariance of safe set using \textit{Controller Margin}\footnote{An additional term compare with Eq.~\eqref{eq:cbc_lemma} added to the right-hand side in Eq.~\eqref{eq:cbc_lemma} is called the controller margin.}, which is summarized in the following lemma.

\begin{lemma}\label{lem:sample-data cbf}[Summarized from \cite{breeden2021control}]
Given the control affine system defined in Eq.~\eqref{eq:Affine} where the state is measured discretely (i.e, the sampled states $x_k= x(k dt)$, where $k$ is the time step and $dt$ is the one time step duration and the control law is updated discretely (i.e., $u(t) = u_k(x_k)$), and the safe set $\mathcal{H}$, the admissible control space for any Lipschitz continuous controller
$u_k \in \mathbb{R}^m$ 
at the time step $k$ rendering $\mathcal{H}$ forward invariant is defined as below:
\begin{align}\label{eq:sampledata_cbc_lemma}
    S_\textrm{cbf}(x_k) = \{ u_k \in \mathbb{R}^m &| L_f h(x_k) + L_gh(x_k)u_k \notag \\
    &\geq -\alpha(h(x_k)) + l\delta dt\}
\end{align}
where $x_k$ is the observed state at time step k, $l = l_{L_fh} + l_{L_gh}u_{max} + l_\alpha(h)$ ($l_{L_fh}, l_{L_gh}$ and $l_\alpha(h)$ are the Lipschitz constants of $L_f h, L_g h$ and $\alpha(h)$, respectively), $\delta = \sup_{x \in S, u\in U} ||f(x) + g(x)u||$.
\end{lemma}
This defines the set of controllers for the system whose states are observed discretely to stay in $\mathcal{H}$ over time, especially during $[t_k, t_{k+1})$ where $t_k = k dt$.

\subsection{Kinematic Vehicle Model}
By assuming that the tire slip ratios and slip angles are small, the vehicle's orientation and position evolve
based on its velocity, steering angle, and wheelbase. 
The kinematics of ego vehicle $\mathcal{V}_e$ as an AV can be described by the Kinematic Bicycle Model \cite{polack2017kinematic} as follows:
\begin{equation}\label{eq:kinematic bicycle model}
    \begin{cases}
      \dot X_e &= v_e \cos(\phi_e + \beta_e) \\
      \dot Y_e &= v_e \sin(\phi_e + \beta_e) \\
      \dot v_e &= a_e  \\
      \dot \phi_e &= \frac{v_e}{l} \sin \beta_e  \\
      \beta_e &= \tan^{-1} (\frac{1}{2 \tan \delta_e})
    \end{cases}       
\end{equation}
where $(X_e, Y_e)$ represents the vehicle's longitude and lateral positions.
$\phi_e$ is the vehicle heading angle and $v_e$ is the velocity of the vehicle. $\beta_e$ is the side slip angle at the center of gravity and $a_e$ the acceleration. $\delta_e$ is the front wheel steering angle.

When the vehicle travels on a highway, the environment and the vehicle dynamics may evolve rapidly, requiring quick responses from the control system
~\cite{10107482}. 
Discretizing the kinematic bicycle model allows the computation to be performed efficiently. In this case, the system in Eq.~\eqref{eq:kinematic bicycle model} can be discretized as a linear parameter-varying system using Euler's method. To simplify the discussion, we use the
mapping $\Tilde{f}$ to represent the system:
\begin{equation}
    \mathbf{s}_{e, k+1} = \Tilde{f}(\mathbf{s}_{e, k}, \mathbf{u}_{e, k})
\end{equation}
where $\mathbf{s}_{e, k} = [X_{e, k}, Y_{e, k}, v_{e, k}, \phi_{e, k}]^{\text{T}} \in \mathbb{R}^4$ and $\mathbf{u}_{e, k}=[a_{e, k}, \delta_{e, k}]^{\text{T}} \in \mathbb{R}^2$ are the state and control input of the vehicle $e$ at time step $k$, respectively. In the following sections, we use $\mathbf{P}_{e,k}$ to represent the position $(X_{e,k}, Y_{e,k})$ of the ego vehicle $\mathcal{V}_e$ at time step $k$.

\subsection{Problem Statement}
We consider the scenario where an AV operates on the road and can observe the neighboring vehicles' positions and velocities. For the observed vehicle $\mathcal{V}_i$ at time step $k$:
\begin{align}
    \mathbf{P}_{i,k} =  (X_{i,k}, Y_{i,k}) \notag \\
    \mathbf{V}_{i,k} = (v^{x}_{i,k}, v^{y}_{i,k}) \notag
\end{align}
where $\mathbf{P}_{i,k}, \mathbf{V}_{i,k}$ represent the position and velocity, respectively.
$v^{x}_{i,k}, v^{y}_{i,k}$ are the velocities of the vehicle $i$ in the $x$ and $y$ direction. Both $\mathbf{P}_{i,k}, \mathbf{V}_{i,k}$ are considered as Gaussian random variables with $\mathbf{P}_{i,k} \sim (\hat{\mathbf{P}}_{i,k}, \Sigma^{P}_i)$ and $\mathbf{V}_{i,k} \sim (\hat{\mathbf{V}}_{i,k}, \Sigma^{V}_i)$, where $\hat{\mathbf{P}}_{i,k} \in \mathbb{R}^{2}$ and $\hat{\mathbf{V}}_{i,k} \in \mathbb{R}^{2}$ are the observed location and velocity of the neighboring vehicle, respectively, and $\Sigma^{P}_i \in \mathbb{R}^{2 \times 2}$ and $\Sigma^{V}_i \in \mathbb{R}^{2 \times 2}$ are the diagonal covariance matrix. 
\begin{remark}
   Neighboring vehicles are the vehicles located within a specific distance that reflects the sensing range.
\end{remark}

Motivated by CBF-inspired risk evaluation framework in \cite{lyu2023risk},
we seek to quantify the risk when there is both position and velocity uncertainty and exploit the potential of this risk evaluation framework in safe and courteous driving.

\section{Methods}

\subsection{Control Barrier Functions-inspired Conditional Value at Risk
under Noisy Observations}
The vehicles operating on roads are assumed to be covered by an ellipse \cite{pierson2018navigating, adajania2022multi}, hence motivated by \cite{wang2017safe} the safety set between any pairwise vehicles at time step $k$ can be formulated as follows:
\begin{equation} \label{eq:safe set}
    \mathcal{H}^{P_{k}}_{e,i} = \{(\mathbf{P}_{e,k}, \mathbf{P}_{i,k}) | h_{e,i}(\mathbf{P}_{k}) \geq 0 \}
\end{equation}
\begin{equation}
    h_{e,i}(\mathbf{P}_{k}) = (X_{e,k} - X_{i,k})^2 + (\frac{Y_{e,k} - Y_{i,k}}{\tau})^2 - \mathcal{D}_s^2
    \label{eq:ellipse}
\end{equation}
where $\mathbf{P}_{k}$
is the joint position space of the ego vehicle $e$ and its neighboring vehicle $\forall i$ at the time step $k$, $\tau$ is the scaling factor and $\mathcal{D}_s$ is the safety distance.

Inspired by Lemma~\ref{lem:sample-data cbf}, we use the function\footnote{In this paper, we select the class $\kappa$ function $\alpha(h)=\gamma h$ as the same setting in \cite{zeng2021safety, lyu2023risk}, where $\gamma \in \mathbb{R} \ge 0$.} $H_{e,i}(\mathbf{P}_{k}, \dot{\mathbf{P}}_k)$ to represent the severity of the potential pairwise collision, which is defined as:
{\small
\begin{align}\label{eq:H_e,i}
    H_{e, i}(\mathbf{P}_{k}, \dot{\mathbf{P}}_k) = -\dot{h}_{e,i}(\mathbf{P}_{k}, \dot{\mathbf{P}}_k) - \gamma h_{e,i}(\mathbf{P}_{k})+ l_H\delta_H dt
\end{align}}
where $l_{H} = l_{L_fh_{e,i}} + l_{L_gh_{e,i}}u_{max} + l_\alpha(h_{e,i})$ with $\delta_{H} = \sup_{\mathbf{P}_{e,k} \in \mathcal{H}^{P_k}_{e,i}} |\dot{\mathbf{P}}_{e,k}|$ and $dt$ is the duration of one time step (the dependence of $h_{e,i}$ is omitted for notation simplicity.). 

With the Gaussian distributed observed positions and velocities of the neighboring vehicles, $H_{e,i}(\mathbf{P}_{k}, \dot{\mathbf{P}}_k)$ is distributed as a noncentral chi-square distribution. Using the Delta method \cite{ver2012invented}, we can approximate the distribution of $H_{e,i}(\mathbf{P}_{k}, \dot{\mathbf{P}}_k)$ as a Gaussian distribution where the mean and variance are estimated with the ﬁrst-order Taylor series expansion, i.e., $H_{e,i}(\mathbf{P}_{k}, \dot{\mathbf{P}}_k) \sim \mathcal{N}(\boldsymbol{\mu}_{H_{e,i}}, \Sigma_{H_{e,i}})$ with:
\begin{align}
\boldsymbol{\mu}_{H_{e,i}} &= -\dot{h}_{e,i}(\hat{\mathbf{P}}_{k}, \dot{\hat{\mathbf{P}}}_k) - \gamma h_{e,i}(\hat{\mathbf{P}}_{k})+ l_{H}\delta_{H} dt
\\
\Sigma_{H_{e,i}} &= \nabla  H_{e, i}(\boldsymbol{\mu}_{H_{e,i}})^{\text{T}} \Sigma_{e,i} \nabla  H_{e, i}(\boldsymbol{\mu}_{H_{e,i}})
\end{align}
where $\Sigma_{e,i} \in \mathbb{R}^{8 \times 8}$ is the covariance of the $[\mathbf{P}_{e,k}, \mathbf{P}_{i,k}, \dot{\mathbf{P}}_{e,k}, \dot{\mathbf{P}}_{i,k}]$.

We employ conditional value-at-risk\footnote{With a user-defined confidence level $\alpha$, $\mathrm{CVaR}$ is the expected cost in the worst $\alpha$-percentile of the distribution function $H_{e,i}(\mathbf{P}_{k}, \dot{\mathbf{P}}_k)$.} (CVaR) \cite{hsu2023interpretable} to quantify the ego vehicle's perceived risk from potential pairwise collision and observation uncertainty:
\begin{align}\label{eq:pairwise cvar}
\mathrm{CVaR}_{\alpha}(H_{e,i}):&=\frac{\int_{1-\alpha}^{1}\mathrm{V}\mathrm{a}\mathrm{R}_{1-\lambda}(H_{e, i})d\lambda}{\alpha} \notag \\
&= E(H_{e,i}|H_{e, i} \geq \mathrm{V}\mathrm{a}\mathrm{R}_{\alpha}(H_{e, i}))
\end{align}
where $\mathrm{V}\mathrm{a}\mathrm{R}_{\alpha}(H_{e, i}) := \text{inf}\{H \in \mathbb{R} | \mathrm{Pr}(H_{e, i} \geq H) \leq \alpha\}$, and $\alpha$ is a user-defined confidence level. $\mathrm{Pr}(\cdot)$ indicates the probability of an event. To simplify the notation, we use $H_{e,i}$ to represent $H_{e,i}(\mathbf{P}_{k}, \dot{\mathbf{P}}_k)$ by omitting the dependence on $\mathbf{P}_{k}$ and $\dot{\mathbf{P}}_k$.

\begin{theorem}\label{theorem:safety}
Consider the safe set $\mathcal{H}^{P_k}_{e,i}$ defined in Eq.~\eqref{eq:safe set}. Let $l_{H} = l_{L_fh_{e,i}} + l_{L_gh_{e,i}}u_{max} + l_\alpha(h_{e,i})$, $\delta_{H} = \sup_{\mathbf{P}_{e,k} \in \mathcal{H}^{P_k}_{e,i}} |\dot{\mathbf{P}}_{e,k}|$, and with a user-defined confidence level $\alpha$. If the ego vehicle is in a safe state at the initial time step, then the constraint defined in Eq.~\eqref{eq:safety constraint} at time step $k$ will guarantee the pairwise vehicle-to-vehicle safety with at least $1-\alpha$ probability during $[t_k, t_{k+1})$.
\begin{equation} \label{eq:safety constraint}
\mathrm{CVaR}_{\alpha}(H_{e, i}) \leq 0.
\end{equation}
\end{theorem}

\begin{proof}
    To simplify the notation, let us denote $L = \{H_{e,i} \geq \mathrm{VaR}_{\alpha}(H_{e,i})\}$. 
    Since $\mathrm{CVaR}_{\alpha} \leq 0$ suggests that the mean of the worst $\alpha$ portion of the outcomes for $H_{e,i}$ is negative, we can get $\mathrm{VaR}_{\alpha}(H_{e,i}) \leq 0$. Assume for contradiction that $\mathrm{VaR}_{\alpha} \geq 0$. Then according to the definition, every element in the set $L$ is larger than $0$, which contradicts the assumption that $\mathrm{CVaR}_{\alpha}(H_{e,i}) = E(H_{e,i}|L) \leq 0$.

    Given $\mathrm{VaR}_{\alpha}(H_{e,i}) \leq 0$ and from the definition of $\mathrm{VaR}$, we have $\mathrm{Pr}(H_{e,i} \geq \mathrm{VaR}_{\alpha}(H_{e,i})) \leq \alpha$, we can get that $\mathrm{Pr}(H_{e,i} \geq 0) \leq \alpha$. Then it is straightforward that:
    \begin{equation}
        \mathrm{Pr}(H_{e,i} \leq 0)=1-\mathrm{Pr}(H_{e,i}\geq 0) \geq 1- \alpha. 
    \end{equation}
    According to the definition of $H_{e,i}$ in Eq.~\eqref{eq:H_e,i} and using Lemma~\ref{lem:sample-data cbf}, we can conclude that the vehicle-to-vehicle safety can be guaranteed with at least $1-\alpha$ probability. 
\end{proof}

\begin{remark} \label{remark:time_horizion}
Note that Eq.~\eqref{eq:pairwise cvar} defines the probabilistic safety for a single time step. When extending the consideration to cover the high-probability safety throughout the entire trajectory, the step-wise probability for the vehicle-to-vehicle safety is lower bounded by $(1-\alpha)^{k_t}$, where $k_t$ is the time steps of the entire trajectory. Readers are referred to \cite{zhu2019chance,luo2020multi} for more detailed discussions.
\end{remark}

Based on the pairwise risk evaluation defined in Eq.~\eqref{eq:pairwise cvar}, we further define the resultant perceived risk $R_{e,k}$ of the ego vehicle $\mathcal{V}_e$ on the highway at time step $k$ as:
\begin{equation}\label{eq:risk_evaluation}
    R_{e,k} = \max  \{\mathrm{CVaR}_{\alpha}(H_{e,i}) , \forall i \neq e \}.
\end{equation}

\begin{remark}
By assuming optimal behavior from both human drivers and automated vehicles, it is reasonable to conclude that the vehicles in the environment are well-spaced \cite{huang2021driving}. Consequently, we utilize the maximum risk derived from all pairwise risks to represent the perceived risk from neighboring vehicles. 
This approach is reasonable by the fact that in real-world scenarios, human drivers tend to focus more on the vehicles that pose the greatest perceived risk. 
A similar strategy using maximum risk to represent the risk perceived by the ego vehicle has been employed in \cite{hsu2023interpretable}.
\end{remark}

\subsection{Ego-perceived Risk Map for Highway Decision Making}

Using the CBF-inspired risk evaluation framework, we can construct a comprehensive risk map that enables deeper insights into how perceived risk is aggregated, enhancing the decision-making processes for the ego vehicle.

We employ a three-lane highway as a running example with two vehicles running on the road where the red vehicle is the ego vehicle $\mathcal{V}_e$ and the blue one $\mathcal{V}_b$ is the observed neighboring vehicle. To construct the risk map depicted in Fig.~\ref{fig:highway-riskmap}, we counterfactually place the ego vehicle at various points across the map and evaluate the perceived risk at each point using Eq.~\eqref{eq:risk_evaluation}. This evaluated risk is then attributed to the corresponding point on the map and visualized using a gradient color scheme to explicitly depict risk variations.

\begin{figure}
\centering  
\begin{subfigure}[b]{0.46\textwidth}
\centering
  \includegraphics[width=\linewidth]{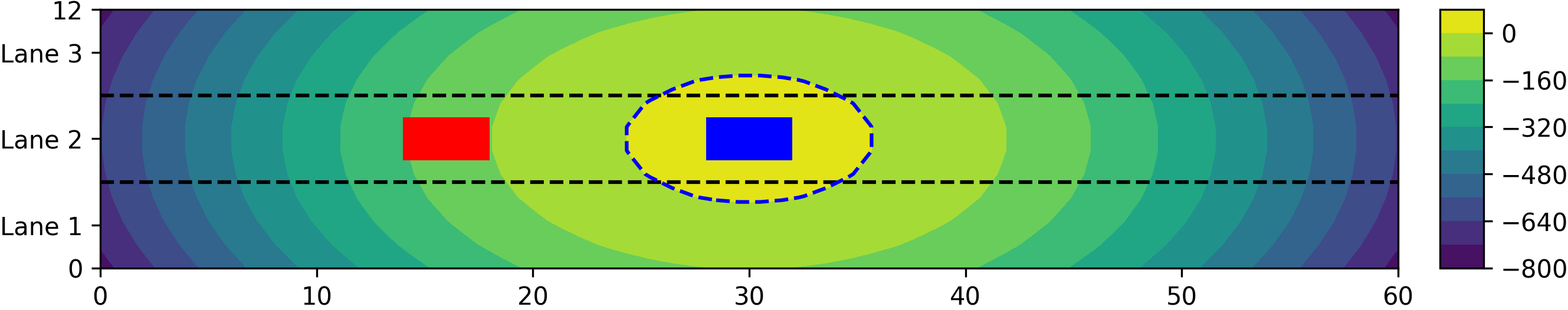}
 \caption{Velocities of $\mathcal{V}_e$ and $\mathcal{V}_b$ are $(15,0) $and $(15,0)$, respectively.}
  \label{fig:suba}
\end{subfigure} \\
\begin{subfigure}[b]{0.46\textwidth}
\centering
  \includegraphics[width=\linewidth]{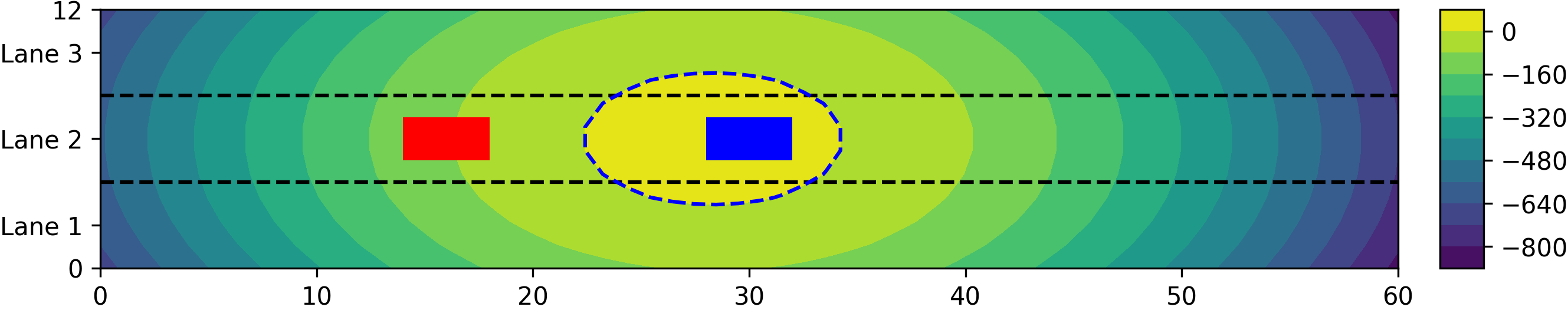}
 \caption{Velocities of $\mathcal{V}_e$ and $\mathcal{V}_b$ are $(15,0) $and $(10,0)$, respectively.}
  \label{fig:subb}
   \end{subfigure} \\
\begin{subfigure}[b]{0.46\textwidth}
\centering
  \includegraphics[width=\linewidth]{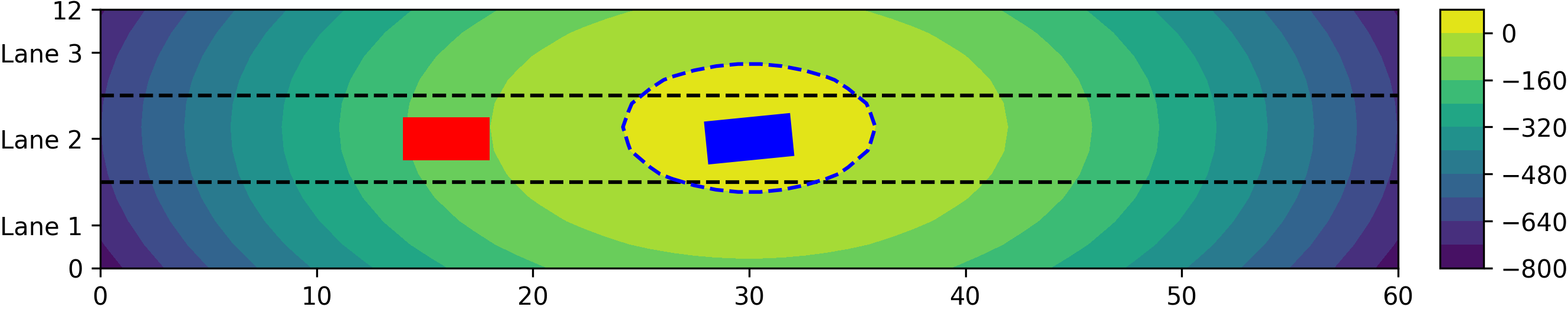}
 \caption{Velocities of $\mathcal{V}_e$ and $\mathcal{V}_b$ are $(15,0) $and $(15,1.5)$, respectively.}
  \label{fig:subc}
\end{subfigure}
\caption{\small Running example for visualizing the CBF-inspired risk map, demonstrating how risk is aggregated when considering various velocity configurations' influence on the ego-perceived risk map using Eq.~\eqref{eq:risk_evaluation}. The red vehicle represents the ego vehicle, and the blue one is the observed neighboring vehicle. The area within the dashed blue contour indicates the region where the evaluated risk is larger than $0$.}
\label{fig:highway-riskmap}
\vspace{-0.5cm}
\end{figure}

Using the risk maps depicted in Fig.~\ref{fig:highway-riskmap}, we illustrate how risk is aggregated based on various velocity configurations and their influence on the ego-perceived risk map. In all three subfigures, the positions of the ego vehicle $\mathcal{V}_e$ and the neighboring vehicle $\mathcal{V}_b$
are the same.
This consistent positioning allows us to isolate and analyze the effects of velocity changes on the perceived risk without the confounding factor of varying spatial relationships.
The state and velocity noise covariance are both set as diag$\{0.1, 0.1\}$ and the confidence level is set as $\alpha = 0.1$. The dashed blue contour in every subfigure delineates the area where the evaluated risk exceeds $0$.
In Fig.~\ref{fig:suba}, both $\mathcal{V}_e$ and $\mathcal{V}_b$ are traveling at a velocity of $(15,0)$, and it is obvious that the closer the ego vehicle is to the neighboring vehicle, the larger the perceived risk. Fig.~\ref{fig:subb} illustrates a shift in risk perception when $\mathcal{V}_b$'s velocity changes to $(10, 0)$, showing a noticeable bias. This is reasonable given that the ego vehicle's velocity exceeds that of the leading vehicle. 

Furthermore, we adjust $\mathcal{V}_b$'s velocity to $(15, 1.5)$, which significantly reduces the safe region in the third lane as shown in Fig.~\ref{fig:subc}. If $\mathcal{V}_e$ intends to overtake $\mathcal{V}_b$, although there appears to be sufficient space in Lane three, the risk map suggests that overtaking in Lane one is a preferable strategy. This decision not only offers a safer alternative but also exhibits courteous behavior towards $\mathcal{V}_b$.

\begin{remark}
This discussion focuses solely on the impact of the observed vehicle's velocity on the risk map, serving to illustrate the effectiveness of the proposed CBF-inspired risk evaluation method. In the context of highway scenarios, velocity is the predominant factor. However, other elements e.g. the safety radius, 
parameter $\gamma$, or
noise distribution
also significantly influence the ego-perceived risk map.
\end{remark}

Based on these observations, it is reasonable to conclude that: 1) the risk map based on the risk evaluation method is applicable and illustrates varying performance under different positions and motions, and 2) it shows the potential for aiding AV's decision-making for safe and courteous driving.

\subsection{Courteous Model Predictive Control}
Inspired by the analysis of the CBF-inspired risk map and Theorem \ref{theorem:safety},  
to realize safe and courteous driving performance, 
we integrate the perceived risk by ego vehicle into the MPC framework. 

For a MPC framework, a general task-related cost function $J(\mathbf{s}_{e, k})$ is defined as below:
\begin{align}\label{eq:MPC cost function}
    J(\mathbf{s}_{e, k})&= \sum_{t=0}^{N-1}\bigl(l_\text{S}(\mathbf{s}_{e,k+t|k}, \mathbf{u}_{e,k+t|k}) \bigr) \notag \\
    &+l_{\text{T}}(\mathbf{s}_{e,k+N|k}, \mathbf{u}_{e,k+N|k})
\end{align}
where $N$ is the prediction horizon, $l_\text{S}(\mathbf{s}_{e,k+t|k}, \mathbf{u}_{e,k+t|k})$ and $l_{\text{T}}(\mathbf{s}_{e,k+N|k}, \mathbf{u}_{e,k+N|k})$ represent the stage cost and terminal cost, respectively.
$\mathbf{s}_{e,k+t|k}$ denotes the state vector at time step $t+k$ predicted at time step $k$ obtained by starting from the current state $\mathbf{s}_{e,k}$ by applying the input sequence $\mathbf{u}_{k:k+N-1|k}$. To realize safe and courteous driving, we further integrate the risk perceived by the ego vehicle to the cost function and constraints to synthesize
the MPC controller, which is defined as:
\begin{subequations}\begin{align}
    & \min_{\mathbf{u}_{k:k+N-1|k}}  J(\mathbf{s}_{e, k}) + \sum_{t=0}^{N}\bigl(P_{\text{S}} R_{e,k+t|k} \bigr) \tag{16} \label{eq:courteous MPC}  \\
     s.t. \quad &\mathbf{s}_{e,k+t+1|k} = \Tilde{f}(\mathbf{s}_{e,k+t|k}, \mathbf{u}_{e,k+t|k}), t=0,\cdots, N-1 \notag \\
     & \mathbf{s}_{e,k|k} = \mathbf{s}_{e,k} \notag \\
    & R_{e,k+t|k} \leq 0, t=0,\cdots, N-1 \label{eq:riskconstraint}  \\
    & \mathbf{s}_{min} \leq \mathbf{s}_{e,k+t|k} \leq \mathbf{s}_{max}, t=0,\cdots, N-1  \notag \\
    & \mathbf{u}_{min} \leq \mathbf{u}_{e,k+t|k} \leq \mathbf{u}_{max}, t=0,\cdots, N-1 \notag
\end{align}\end{subequations}
where $P_{\text{S}}$ is a user-defined weight parameter and $R_{e,k+t|k}$ is the predicted risk at time step $k$. $\mathbf{s}_{min}$ and $\mathbf{s}_{max}$ represent the lower and upper bounds of the state constraints. $\mathbf{u}_{min}$ and $\mathbf{u}_{max}$ represent the lower and upper bounds of the control inputs. By the constraint of Eq.~\eqref{eq:riskconstraint}, we can guarantee safety with at least $1-\alpha$ probability. By integrating the risk evaluation into the cost function, the ego vehicle may make decisions to reduce the perceived risk to yield more space for the neighboring vehicles.

\section{Experiments and Discussion}
To validate the effectiveness of the proposed Courteous MPC,
the parameter analysis experiment is conducted first to discuss the influence of weight parameter $P_{\mathrm{S}}$ on the ego vehicle courtesy performance.
Then, the simulation experiments on interaction with both IDM vehicles and the NGSIM US101 dataset are conducted to validate the effectiveness of our proposed Courteous MPC. We employ the do-mpc library \cite{fiedler2023mpc} to formulate the MPC controller and utilize IPOPT \cite{wachter2006implementation} as the nonlinear solver.

\subsection{Parameter Analysis}\label{sec: 4.1}
To exploit the influence of the weight parameter $P_{\mathrm{S}}$ on the ego vehicle performance, we conducted a cruise driving task with varying $P_{\mathrm{S}}$ in the highway-env~\cite{highway-env} simulation environment, where the behavior of surrounding vehicles is controlled by IDM \cite{treiber2000congested}. The state and velocity noise covariances are both set as diag$\{0.1,0.1\}$ with the confidence level set as $\alpha = 0.1$. The target speed of the ego vehicle is $15\; m/s$.
As depicted in Fig.~\ref{fig:influence of P}, with the increase of $P_{\mathrm{S}}$ in Eq.~\ref{eq:courteous MPC}, the minimum distance between the ego vehicle and neighboring vehicles also increases. This indicates that the penalty of the risk term allows more space around the ego vehicle, enhancing comfort for surrounding vehicles.
Additionally, we can observe that even with a lower
$P_{\mathrm{S}}$, a reasonable space is still maintained to ensure safe driving.

\begin{figure}[t]
    \centering
    \includegraphics[width=2.5 in]{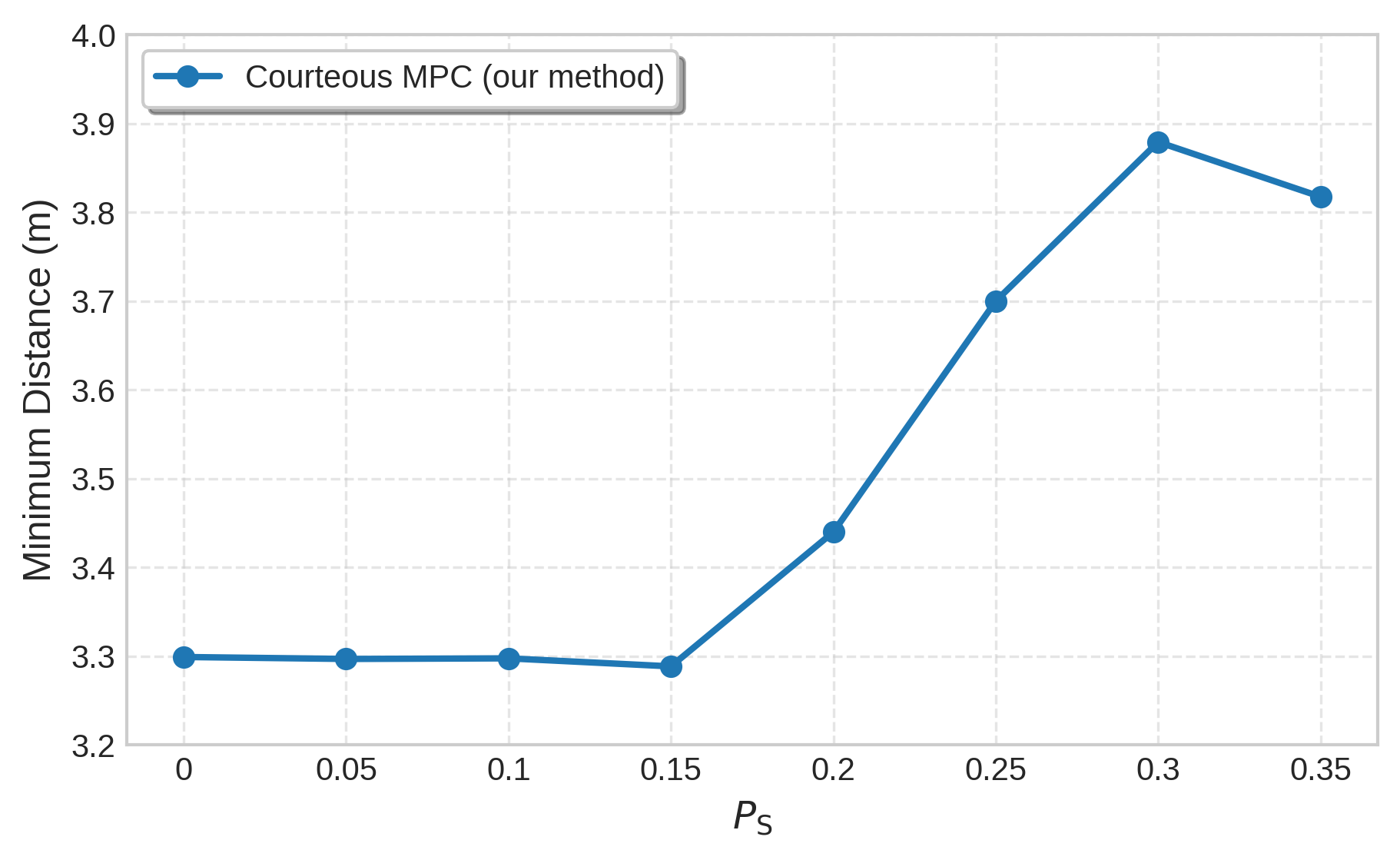}
    \caption{Influence of $P_{\mathrm{S}}$ on the minimum distance between the ego vehicle and neighboring IDM vehicles.}
    \label{fig:influence of P}
    \vspace{-0.55 cm}
\end{figure}

\begin{figure*}[!ht]
\centering
\begin{subfigure}[b]{0.45\textwidth}
  \includegraphics[width=\linewidth]{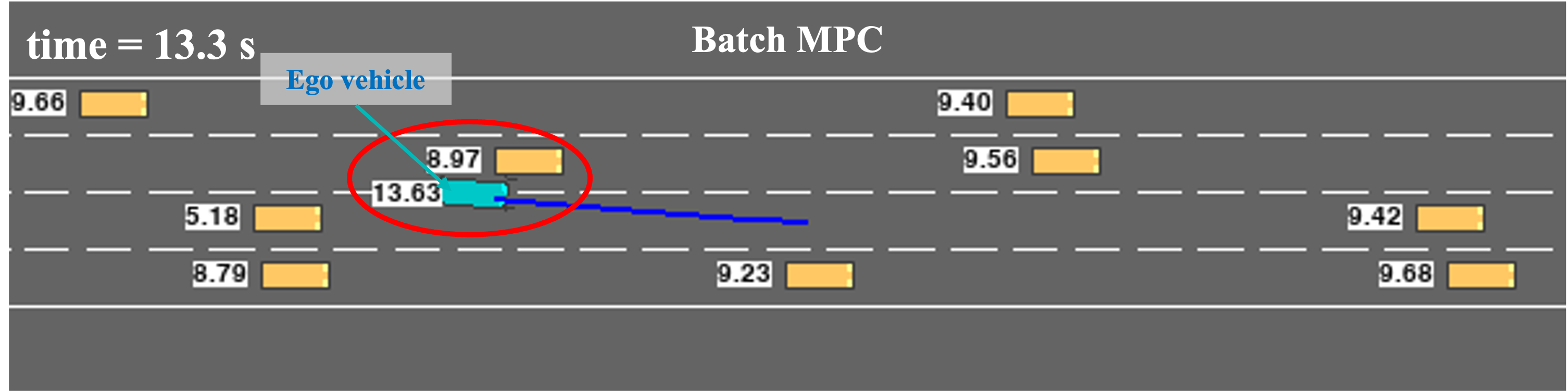}
 \caption{}
  \label{fig:batchmpc1}
  \end{subfigure}
\begin{subfigure}[b]{0.45\textwidth}
  \includegraphics[width=\linewidth]{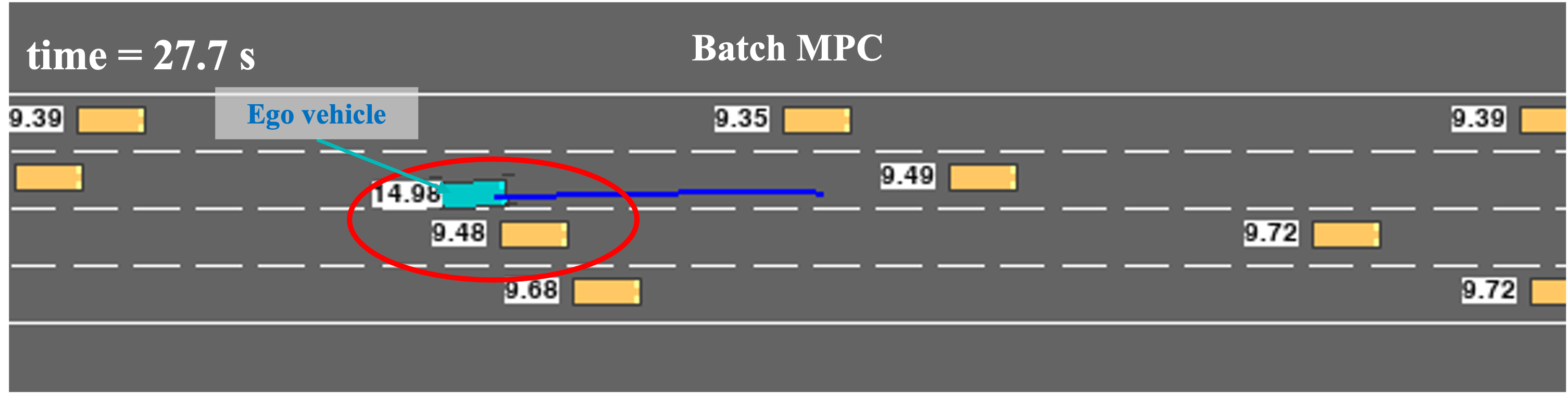}
 \caption{}
  \label{fig:batchmpc2}
\end{subfigure}
\begin{subfigure}[b]{0.45\textwidth}
  \includegraphics[width=\linewidth]{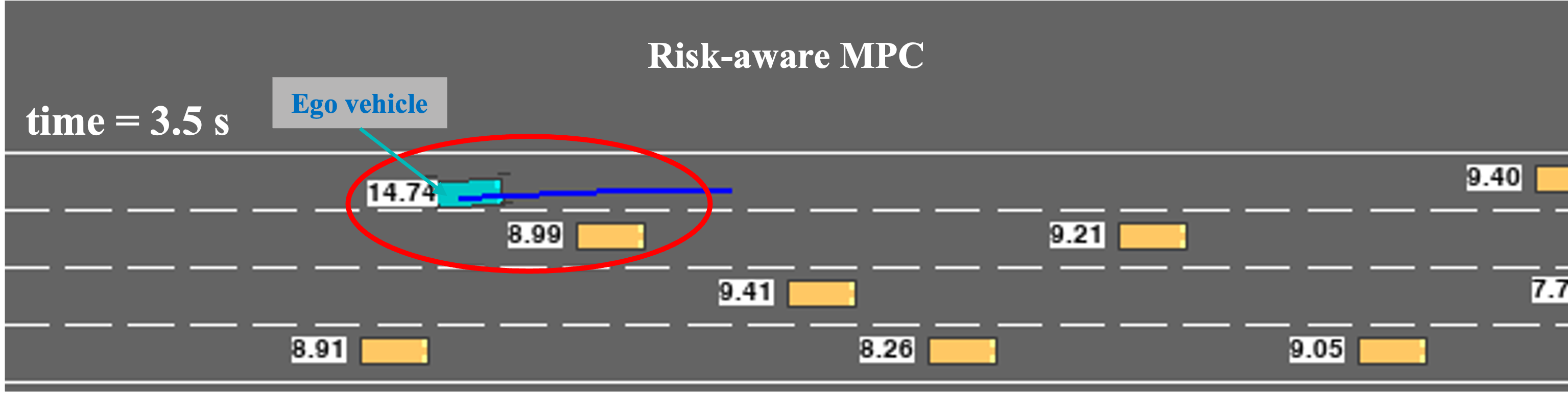}
 \caption{}
  \label{fig:risk-aware mpc1}
   \end{subfigure} 
\begin{subfigure}[b]{0.45\textwidth}
  \includegraphics[width=\linewidth]{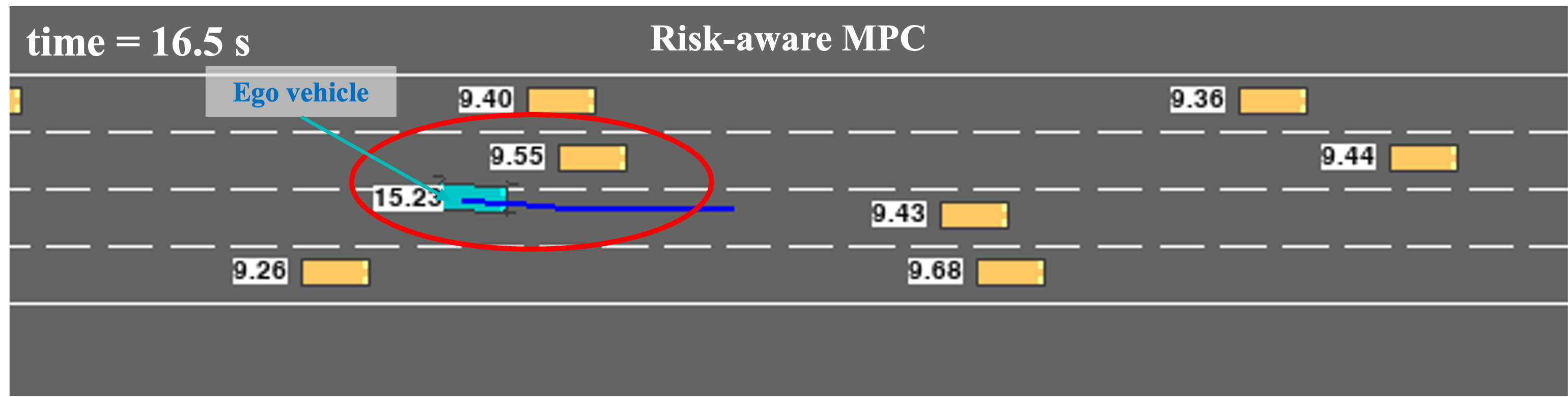}
 \caption{}
  \label{fig:risk-aware mpc2}
\end{subfigure}
\begin{subfigure}[b]{0.45\textwidth}
  \includegraphics[width=\linewidth]{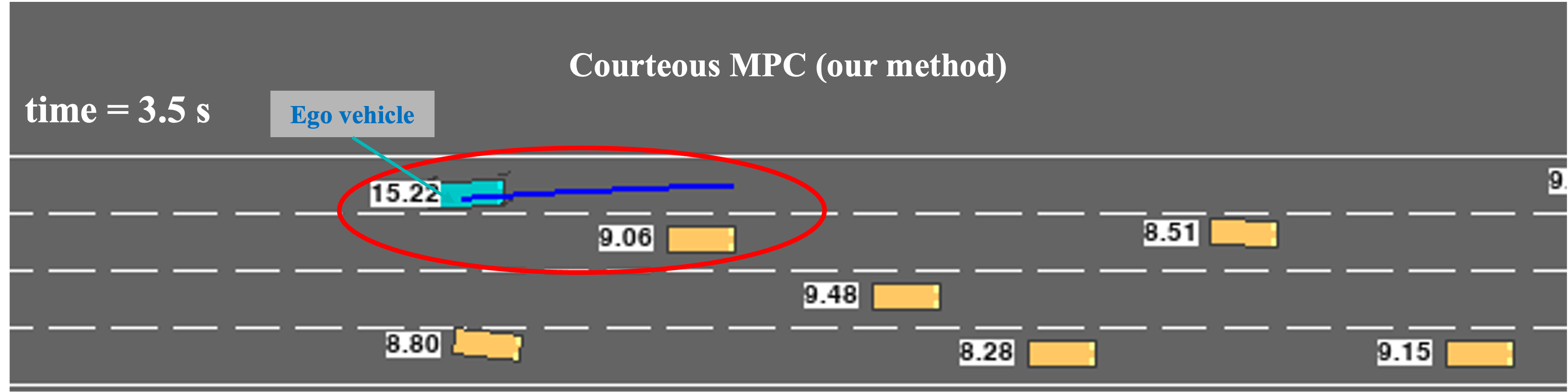}
 \caption{}
  \label{fig:courteous mpc1}
   \end{subfigure} 
\begin{subfigure}[b]{0.45\textwidth}
  \includegraphics[width=\linewidth]{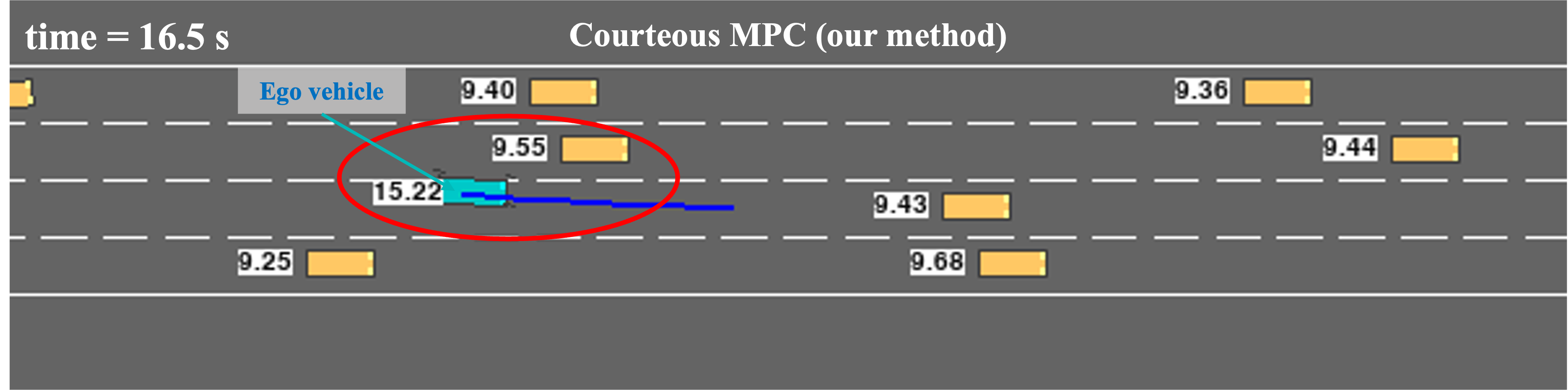}
 \caption{}
  \label{fig:courteous mpc2}
\end{subfigure}
  \caption{\small Performance comparison of the simulation experiments in the highway-env environment, depicting two different examples of overtaking behavior for each of three methods. Fig.~\ref{fig:batchmpc1} and Fig.~\ref{fig:batchmpc2} depict the results for Batch MPC, Fig.~\ref{fig:risk-aware mpc1} and Fig.~\ref{fig:risk-aware mpc2} for Risk-aware MPC, and Fig.~\ref{fig:courteous mpc1} and Fig.~\ref{fig:courteous mpc2} for Courteous MPC. The red ellipse in each figure highlights the ego vehicle and its overtaking target. And the value shown behind every vehicle is its speed.}
  \label{fig:performance_comparsion}
  \vspace{-0.6 cm}
\end{figure*}

\subsection{Simulation Experiment on IDM Vehicles}
To further validate the performance of the proposed Courteous MPC, three methods including Courteous MPC, Risk-aware MPC and Batch MPC \cite{adajania2022multi} are conducted for the cruise control task
in the highway-env simulation environment \cite{highway-env}.
For clarity, the controller is designated as "Risk-aware MPC" when the risk term is removed from the cost function.
For Courteous MPC, we use the same parameter setting as Section.~\ref{sec: 4.1} and set $P_{\mathrm{S}}=0.25$.

\begin{figure}
    \centering
    \includegraphics[width=2.4 in]{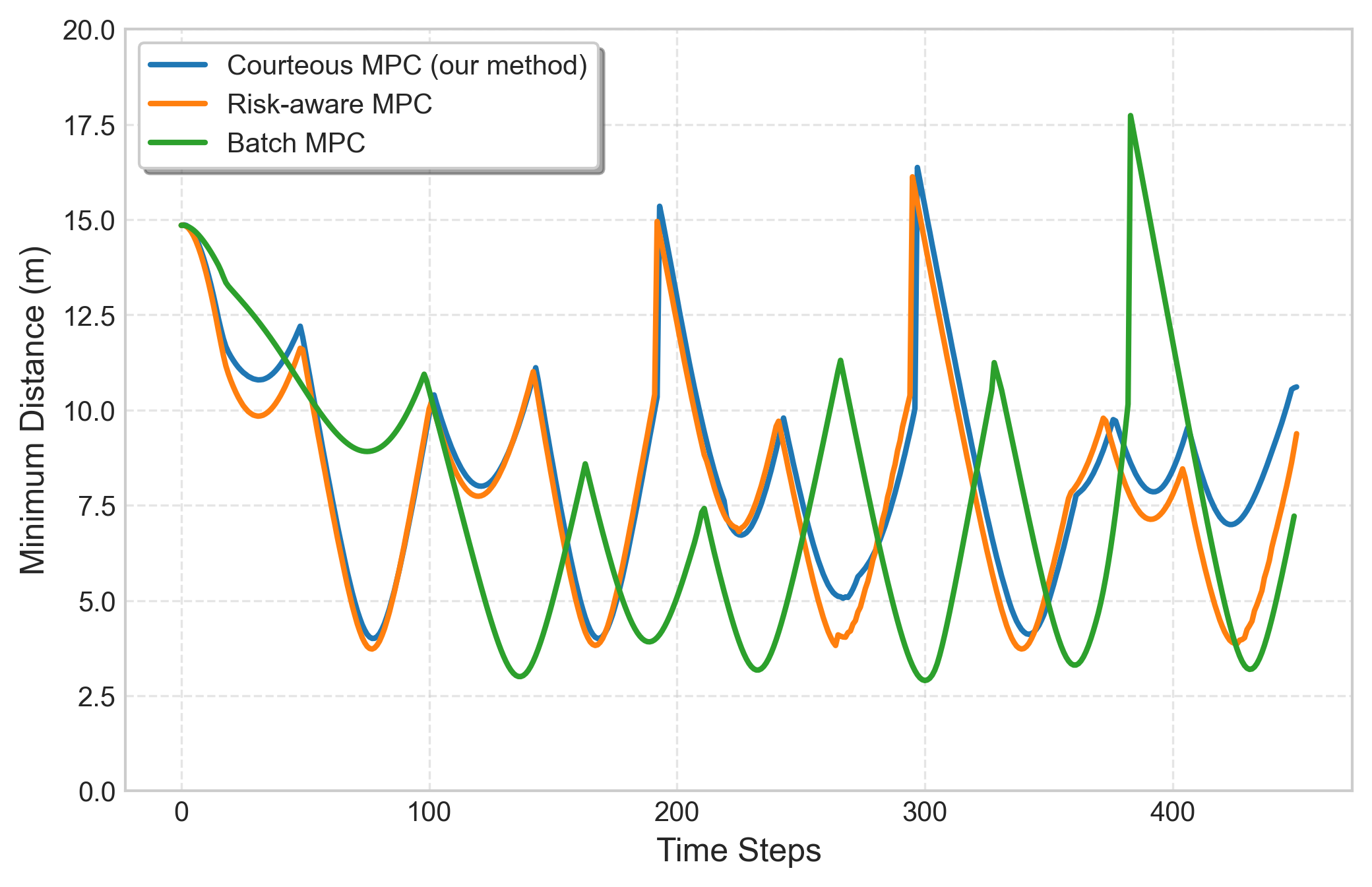}
    \caption{\small Minimum distance between the ego vehicle and the neighboring IDM vehicles over time.}
    \label{fig:min_distance}
    \vspace{-0.55 cm}
\end{figure}

As depicted in Fig.~\ref{fig:performance_comparsion}, we provide snapshots demonstrating two examples of overtaking behavior for each method. Both Risk-aware MPC and Courteous MPC initiate overtaking maneuvers earlier than Batch MPC, which supports the efficacy of our CBF-inspired risk evaluation framework in preventing collisions proactively.
Compared to Risk-aware MPC, Courteous MPC initiates overtaking maneuvers even earlier. This suggests that integrating the risk term into the cost function allows more space around neighboring vehicles, thereby enhancing comfort and safe margins. 
Additionally, Fig.~\ref{fig:min_distance} shows the minimum distance maintained by the ego vehicle from IDM-controlled vehicles over time.
This analysis also reveals that 
Courteous MPC consistently allows greater space around the ego vehicle, enhancing comfort for surrounding vehicles\footnote{The simulation details can be found at \url{https://youtu.be/s5aeT6MPhio}
}.

To prove the effectiveness of Courteous MPC, we compare the cruise task performance of the three methods as shown in Table \ref{tab:task performance}. 
Performance metrics include the average speed of the ego vehicle and the total travel
distance in the longitudinal direction. 
The data reveals that Courteous MPC significantly outperforms the other two methods on both metrics. Notably, the relatively lower accuracy of Batch MPC aligns with the findings reported in \cite{zheng2024barrier}. 
These improvements suggest that providing additional space does not detrimentally affect overall task performance. Instead, it likely enables the ego vehicle to seek more open areas, allowing it to achieve higher speeds rather than confining itself
within narrow spaces.

\begin{table}[ht]
\vspace{-0.2 cm}
\centering
\caption{\small Task Performance Comparison}
\label{tab:task performance}
\begin{tabular}{|c|c|c|}
\hline
 & Average speed ($m/s$) & \thead{Longitude distance ($m$) \\within 30 seconds }\\ \hline
Courteous MPC & 14.96 & 446.09 \\ \hline
Risk-aware MPC & 14.89 & 441.89 \\ \hline
Batch MPC & 13.80 & 414.60 \\ \hline
\end{tabular}
\vspace{-0.2 cm}
\end{table}

To further validate the robustness of the Courteous MPC, we conduct $10$ trials for the three methods and plot the average minimum distances with error bars, as shown in Figure~\ref{fig:average_min_distance}. For every trial, the positions of the vehicles are randomly initiated. The results corroborate our earlier findings, indicating that Courteous MPC consistently allows for greater spacing around surrounding vehicles, thereby facilitating more courteous interactions on the road.

\subsection{Simulation Experiment on NGSIM Dataset}
To further verify the performance of the Courteous MPC on real-world data set, we conduct the simulation experiment in the 
NGSIMENV simulation environment \cite{huang2021driving} with Next Generation Simulation (NGSIM) US101 dataset \cite{alexiadis2004next}, where a human-driven vehicle is replaced by the intelligent vehicle empowered by the Courteous MPC. As we can observe from Fig.~\ref{fig:performance_ngsim}, the Courteous MPC can successfully 
travel through
\textit{heterogeneous} vehicles, which is realized by setting different scaling factors $\tau$ based on the vehicles' sizes in Eq.~\eqref{eq:ellipse}.
Note that in Fig.~\ref{fig:courteousmpc_ngsim2}, the ego vehicle is running on the dashed white lane due to the neighboring two vehicles are running too close to the lane. The best way for the ego vehicle to show courtesy to both two vehicles is to maintain its position on the dashed white lane
\footnote{The simulation details can be found at \url{https://youtu.be/9Cekw5xafWw}
}.

\begin{figure}
    \centering
    \includegraphics[width=2.7 in]{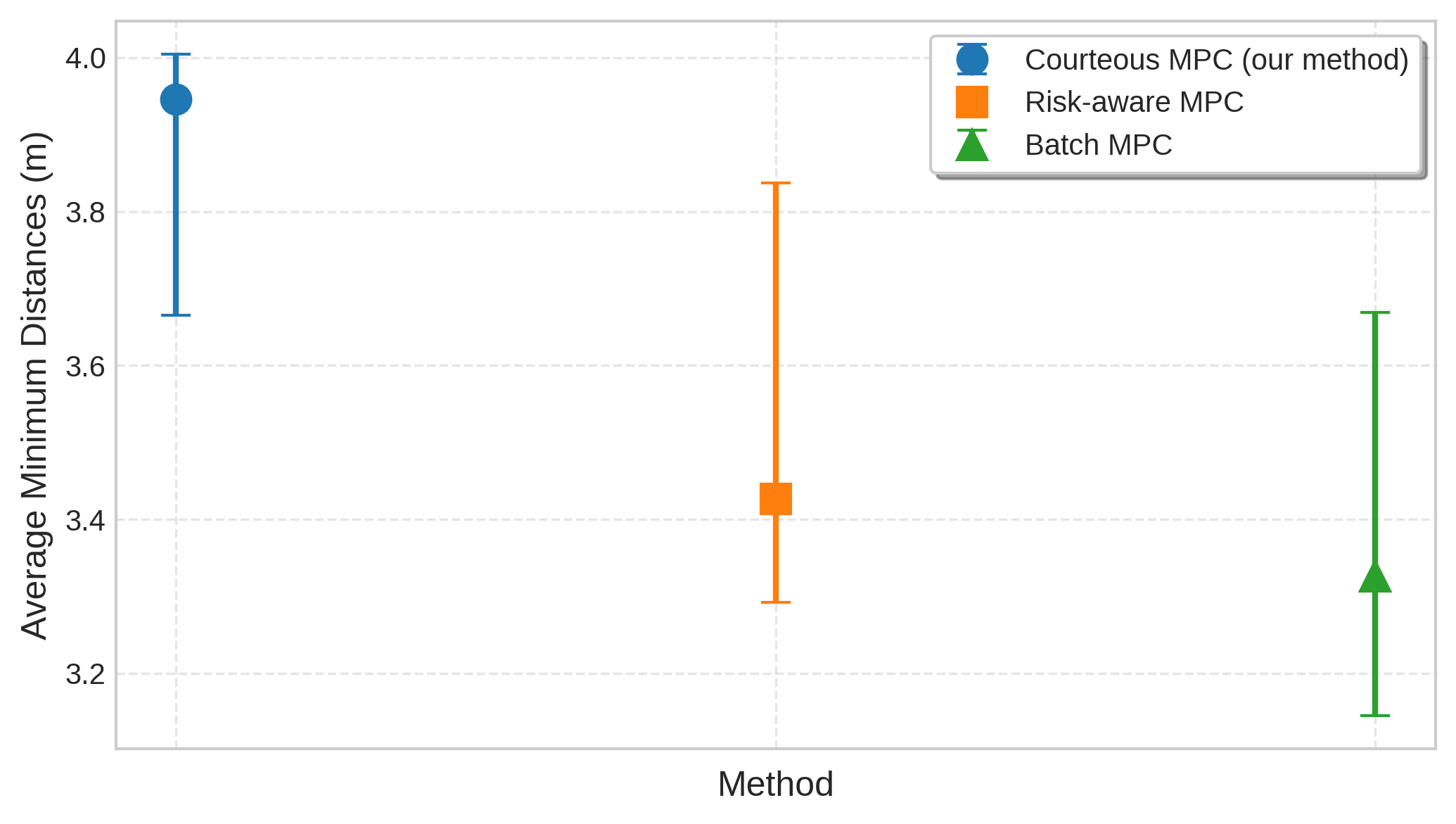}
    \caption{\small Average minimum distance between the ego vehicle and the neighboring IDM vehicles with three methods. The error bar shows the minimum and maximum values.}
    \label{fig:average_min_distance}
    \vspace{-0.6 cm}
\end{figure}

\begin{figure*}[!ht]
\centering
\begin{subfigure}[b]{0.38\textwidth}
  \includegraphics[width=\linewidth]{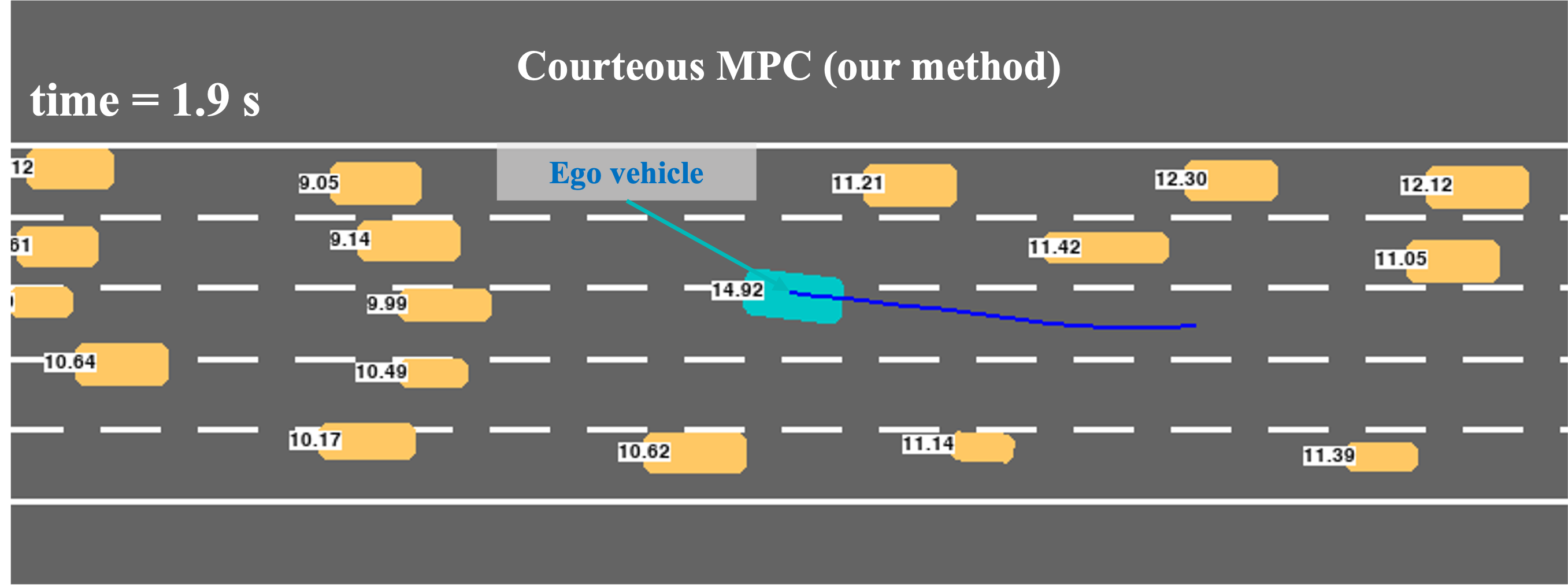}
 \caption{}
  \label{fig:courteousmpc_ngsim1}
  \end{subfigure}
  \qquad
\begin{subfigure}[b]{0.38\textwidth}
  \includegraphics[width=\linewidth]{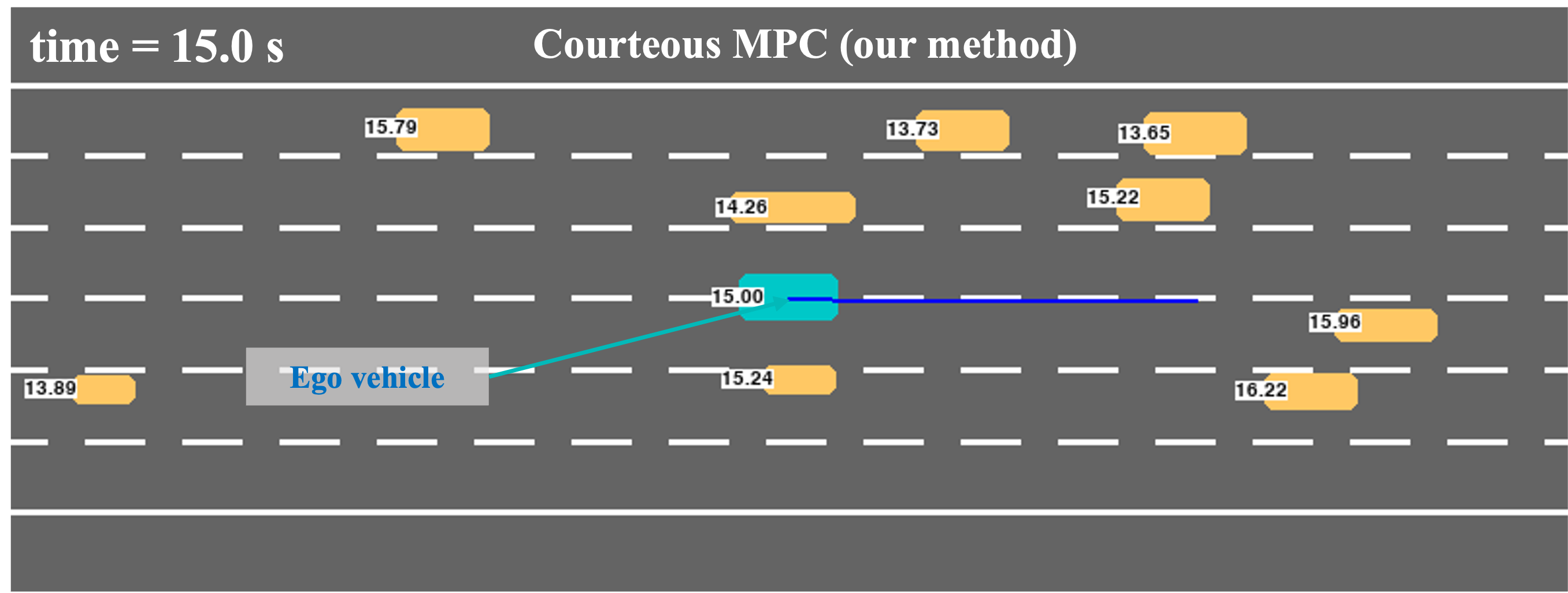}
 \caption{}
  \label{fig:courteousmpc_ngsim2}
\end{subfigure}
\begin{subfigure}[b]{0.38\textwidth}
  \includegraphics[width=\linewidth]{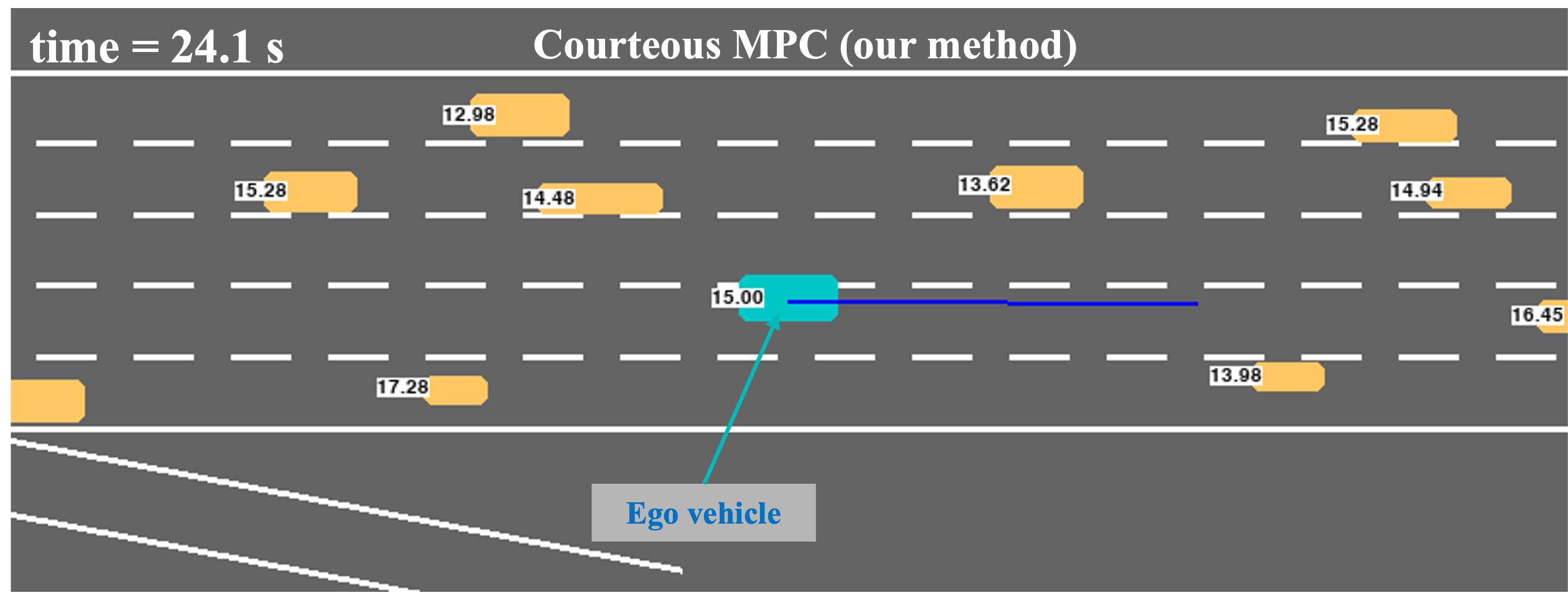}
 \caption{}
  \label{fig:courteousmpc_ngsim3}
   \end{subfigure} 
\qquad
\begin{subfigure}[b]{0.38\textwidth}
  \includegraphics[width=\linewidth]{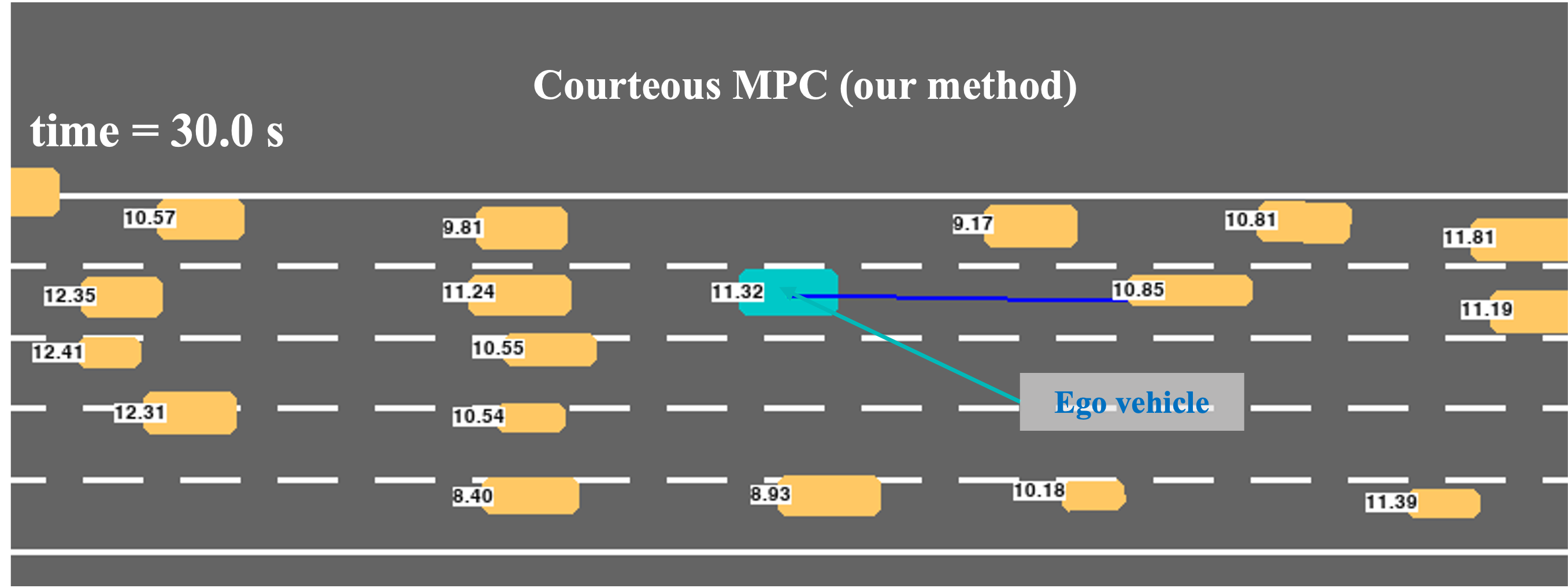}
 \caption{}
  \label{fig:courteousmpc_ngsim4}
   \end{subfigure} 
  \caption{\small Snapshots of performance on NGSIM US101 dataset \cite{alexiadis2004next} using Courteous MPC defined in Eq.~\eqref{eq:courteous MPC}. The value shown behind every vehicle is its speed.}
  \label{fig:performance_ngsim}
  \vspace{-0.75 cm}
\end{figure*}

\section{Conclusions}
In this paper, we proposed an extension of the Control Barrier Functions-inspired risk evaluation framework while considering both noisy observed positions and velocities of neighboring vehicles. The visualization of the risk map suggests that the proposed risk evaluation framework has the potential to aid the decision-making of AVs for safe and courteous driving. The resultant Courteous Model Predictive Control (Courteous MPC) which integrates the risk term into the cost function and constraints shows superior performance through a detailed analysis of the efficacy, efficiency, and robustness using simulation experiments.

\bibliographystyle{IEEEtran}
\bibliography{IEEEexample}

\end{document}